\documentclass[conference,10pt]{IEEEtran}
\IEEEoverridecommandlockouts
\usepackage{float}
\usepackage{amsmath}
\usepackage{amsfonts}
\usepackage{mathtools}
\usepackage{amssymb}
\usepackage{float}
\usepackage{xcolor}
\usepackage{enumerate}
\usepackage{isomath}
\usepackage{bm}
\usepackage{cite}
\usepackage{caption}
\usepackage{subfig}
\usepackage{graphicx}
\usepackage{algorithm,algcompatible,amsmath}
\algnewcommand\INPUT{\item[\textbf{Input:}]}%
\algnewcommand\OUTPUT{\item[\textbf{Output:}]}%
\restylefloat{table}
\usepackage{tikz} 
\usepackage{lipsum}
\newcommand\blfootnote[1]{%
  \begingroup
  \renewcommand\thefootnote{}\footnote{#1}%
  \addtocounter{footnote}{-1}%
  \endgroup
}
\usetikzlibrary{arrows, positioning, automata}
\usetikzlibrary{shapes.geometric, arrows}
\def\xfoo#1^#2\relax#3\valign{%
\mathbf{#1}\ifx\valign#2\valign\else^{\mathbf{#2}}\fi}

\def\KL{{\rm KL}}
\def\JS{{\rm JS}}

\def\EWRM{{\rm EWRM}}

\def\avg{{\rm avg}}

\def\Escr{\mathscr{E}}

\def \Wscr{\mathscr{W}}

\usepackage{amsmath, amssymb, bbm, xspace}
\usepackage{epsfig}
\usepackage{longtable}
\usepackage{color}
\usepackage{mathrsfs}
\usepackage{comment}

\usepackage{courier}

\newtheorem{theorem}{Theorem}[section]
\newtheorem{assumption}{Assumption}[section]

\newtheorem{lemma}{Lemma}[section]

\newtheorem{corollary}[theorem]{Corollary}

\def\bkE{{\rm I\kern-.17em E}}
\def\bk1{{\rm 1\kern-.17em l}}
\def\bkD{{\rm I\kern-.17em D}}
\def\bkR{{\rm I\kern-.17em R}}
\def\bkP{{\rm I\kern-.17em P}}

\def\bkZ{{\bf{Z}}}

\def\bkE{{\rm I\kern-.17em E}}
\def\bk1{{\rm 1\kern-.17em l}}
\def\bkD{{\rm I\kern-.17em D}}
\def\bkR{{\rm I\kern-.17em R}}
\def\bkP{{\rm I\kern-.17em P}}

\makeatletter
\newcommand{\pushright}[1]{\ifmeasuring@#1\else\omit\hfill$\displaystyle#1$\fi\ignorespaces}
\newcommand{\pushleft}[1]{\ifmeasuring@#1\else\omit$\displaystyle#1$\hfill\fi\ignorespaces}
\makeatother

\def\bkZ{{\bf{Z}}}
\def\b12{(\beta_1,\beta_2)}

\newcounter{example}
\renewcommand{\theexample}{\thesection.\arabic{example}}

\newcounter{remark}
\renewcommand{\theremark}{\thesection.\arabic{remark}}

\def\Hscr{\mathscr{H}}

\def\Ebb{\mathbb{E}}
\newlength{\noteWidth}
\long\def\notes#1{\ifinner
{\tiny #1}
\else
\marginpar{\parbox[t]{\noteWidth}{\raggedright\tiny #1}}
\fi\typeout{#1}}

 \def\notes#1{\typeout{read notes: #1}} 

\newcommand{\ie}{i.e.\@\xspace} 
\newcommand{\etal}{et al.\@\xspace} 

\newcommand{\Real}{\ensuremath{\mathbb{R}}}

\def\Ebb{\mathbb{E}}

\def\Ibb{{\mathbb{I}}}

\def\exp{\mathop{\hbox{\rm exp}}}

\def\spose#1{\hbox to 0pt{#1\hss}}

\def\text #1{\hbox{\quad#1\quad}}

\def\Escr{\mathcal{E}}

\def\nthinsp{\mskip -2   mu}

\def\superstar{^{\raise 0.5pt\hbox{$\nthinsp *$}}}
\def\SUPERSTAR{^{\raise 0.5pt\hbox{$*$}}}

\def\lamstarT {\lambda^{\raise 0.5pt\hbox{$\nthinsp *$}T}}

\def\Ascr{{\cal A}}

\def\Pscr{{\cal P}}

\def\Sscr{{\cal S}}

\def\Wscr{{\cal W}}

\def\Zscr{{\cal Z}}

\def\supp{{\rm supp}}

\def\non{\nonumber}

\let\forallnew\forall
\renewcommand{\forall}{\forallnew\ }
\let\forall\forallnew

		\def\bkE{{\rm I\kern-.17em E}}
		\def\bk1{{\rm 1\kern-.17em l}}
		\def\bkD{{\rm I\kern-.17em D}}
		\def\bkR{{\rm I\kern-.17em R}}
		\def\bkP{{\rm I\kern-.17em P}}
		\def\bkY{{\bf \kern-.17em Y}}
		\def\bkZ{{\bf \kern-.17em Z}}
		\def\bkC{{\bf  \kern-.17em C}}

{\begin{list}{}%
         {\setlength{\leftmargin}{#1}}%
         \item[]%
}
{\end{list}}

		\def\bsp{\begin{split}}
		\def\beq{\begin{eqnarray}}
		\def\bal{\begin{align*}}
		\def\bc{\begin{center}}
		\def\be{\begin{enumerate}}
		\def\bi{\begin{itemize}}
		\def\bs{\begin{small}}
		\def\bS{\begin{slide}}
		\def\ec{\end{center}}
		\def\ee{\end{enumerate}}
		\def\ei{\end{itemize}}
		\def\es{\end{small}}
		\def\eS{\end{slide}}
		\def\eeq{\end{eqnarray}}
		\def\eal{\end{align*}}
		\def\esp{\end{split}}
		\def\qed{ \vrule height7.5pt width7.5pt depth0pt}  

	\def\cp2problem#1#2#3#4{\fbox
		 {\begin{tabular*}{0.9\textwidth}
			{@{}l@{\extracolsep{\fill}}l@{\extracolsep{6pt}}l@{\extracolsep{\fill}}c@{}}
				#1 & & $#4 $ 
			\end{tabular*}}}

		\def\bkE{{\rm I\kern-.17em E}}
		\def\bk1{{\rm 1\kern-.17em l}}
		\def\bkD{{\rm I\kern-.17em D}}
		\def\bkR{{\rm I\kern-.17em R}}
		\def\bkP{{\rm I\kern-.17em P}}
		
		\def\bkZ{{\bf{Z}}}

\newcommand {\beeq}[1]{\begin{equation}\label{#1}}
\newcommand {\eeeq}{\end{equation}}
\newcommand {\bea}{\begin{eqnarray}}
\newcommand {\eea}{\end{eqnarray}}

\def\texitem#1{\par\smallskip\noindent\hangindent 25pt
               \hbox to 25pt {\hss #1 ~}\ignorespaces}

\def\bsp{\begin{split}}
		\def\beq{\begin{eqnarray}}
		\def\bal{\begin{align*}}
		\def\bc{\begin{center}}
		\def\be{\begin{enumerate}}
		\def\bi{\begin{itemize}}
		\def\bs{\begin{small}}
		\def\bS{\begin{slide}}
		\def\ec{\end{center}}
		\def\ee{\end{enumerate}}
		\def\ei{\end{itemize}}
		\def\es{\end{small}}
		\def\eS{\end{slide}}
		\def\eeq{\end{eqnarray}}
		\def\eal{\end{align*}}
		\def\esp{\end{split}}
		\def\qed{ \vrule height7.5pt width7.5pt depth0pt}  

\usepackage{amsmath, amssymb, xspace}
\usepackage{epsfig}
\usepackage{longtable}
\usepackage{color}
\usepackage{mathrsfs}
\usepackage{subfig}
\newenvironment{proof}[1][]{{\noindent \textit{ Proof}: }}{\hfill \qed \vspace{3pt}\\ }

\ifCLASSINFOpdf
  \else
 
\fi

\author{\IEEEauthorblockN{Sharu Theresa Jose and Osvaldo Simeone}}
\title{Information-Theoretic Bounds on Transfer Generalization Gap Based on Jensen-Shannon Divergence}
\vspace{-0.1cm}
\begin{document}
\maketitle
\begin{abstract}
In transfer learning, training and testing data sets are drawn from different data distributions.
The transfer generalization gap is the difference between the population loss on the target data distribution and the training loss. The training data set generally includes data drawn from both source and target distributions. This work presents novel information-theoretic upper bounds on the average transfer generalization gap that capture $(i)$ the domain shift between the target data distribution $P'_Z$ and the source distribution $P_Z$ through a two-parameter family of generalized $(\alpha_1,\alpha_2)$-Jensen-Shannon (JS) divergences; and $(ii)$ the sensitivity of the transfer learner output $W$ to each individual sample of the data set $Z_i$ via the mutual information $I(W;Z_i)$. For $\alpha_1 \in (0,1)$, the $(\alpha_1,\alpha_2)$-JS divergence can be bounded even when the support of $P_Z$ is not included in that of $P'_Z$. This contrasts the Kullback-Leibler (KL) divergence $D_{\KL}(P_Z||P'_Z)$-based bounds of Wu \emph{et al.} \cite{wu2020information}, which are vacuous under this assumption. Moreover, the obtained bounds hold for unbounded loss functions with bounded cumulant generating functions, unlike the $\phi$-divergence based bound of Wu \emph{et al.} \cite{wu2020information}. We also obtain new upper bounds on the average transfer excess risk in terms of the $(\alpha_1,\alpha_2)$-JS divergence for empirical weighted risk minimization (EWRM), which minimizes the weighted average training losses over source and target data sets. Finally, we provide a numerical example to illustrate the merits of the introduced bounds.
\end{abstract}
\blfootnote{The authors are with King's Communications, Learning, and Information Processing (KCLIP) lab at the Department of Engineering of King’s College London, UK (emails: sharu.jose@kcl.ac.uk, osvaldo.simeone@kcl.ac.uk).
The authors have received funding from the European Research Council
(ERC) under the European Union’s Horizon 2020 Research and Innovation
Programme (Grant Agreement No. 725731).The authors thank Prof. Tan (NUS) for useful discussions.}
\vspace{-0.6cm}
\section{Introduction}\label{sec:intr}
In conventional learning, data sets for training and testing are drawn from the same underlying data distribution. \textit{Transfer learning} considers the scenario where a learning algorithm trained using a data set drawn from a source data distribution, or \textit{source domain}, is tested on a data set drawn from a generally different target data distribution, or \textit{target domain}.
 The goal of transfer learning is to infer a model parameter $w$ from observation of the data from the source domain and possibly also from target domain, so that it generalizes well on test data from the target domain \cite{torrey2010transfer}. 

The objective of the transfer learner is to minimize the generalization, or population, loss $L_g(w)$, which is the average loss of model parameter $w$ over the test data drawn from the target data distribution. However, this is not available at the learner since the target domain distribution is unknown. Instead, the learner can compute the empirical training loss $L_t(w|Z^M)$ of the parameter $w$ on the data set $Z^M$, which is comprised of data from source and, possibly, target domains. We define the transfer learner as a stochastic mapping $P_{W|Z^M}$ from the input training set to the output space of model parameters. The difference between the generalization loss and the training loss, $\Delta L(w|Z^M)=L_g(w)-L_t(w|Z^M)$, known as the \textit{transfer generalization gap}, is a key metric to evaluate the performance of a transfer learning algorithm. Specifically, if the transfer generalization gap is small, on average or with high probability, the performance of the model parameter $w$ on the training loss can be taken as a reliable estimate of the generalization loss.

Existing works on transfer learning \cite{ben2007analysis,ben2010theory,mansour2009domain,zhang2012generalization} have largely focused on obtaining \textit{high-probability}, probably approximately correct (PAC),  bounds on the transfer generalization gap.
 These bounds have the general form: With probability at least $1-\delta$, with $\delta \in (0,1)$, over the training set $Z^M$, the  bound
$|\Delta L(w|Z^M)| \leq \epsilon$ holds uniformly for all $w \in \Wscr$.
The upper bound $\epsilon$ has been expressed as a function of a distance measure $d(\Sscr,\mathcal{T})$ that quantifies the distributional shift between source ($\Sscr$) and target ($\mathcal{T}$) domains.
  Specifically, the main goal of these studies has been to define appropriate distance measures $d(\Sscr,\mathcal{T})$ that can be estimated from finite data with reasonable accuracy. For example,  Ben \etal in \cite{ben2007analysis} and \cite{ben2010theory} introduce the $d_{\Ascr}$ distance and $\Hscr\Delta \Hscr$-divergence respectively for the 0-1 loss, while Mansour \etal \cite{mansour2009domain} proposed a \textit{discrepancy distance}  that holds for any loss functions. 
   These measures depend on the structural properties of the model class $\Wscr$ through the model complexity measures such as Vapnik-Chervonenkis (VC) dimension and Radmacher complexity.  Similar high probability bounds have also been studied for the optimality gap, \ie, $\Ebb_{P_{W|Z^m}}[L_g(w)]-\min_{w \in \Wscr}L_g(w)$.

In contrast to these prior works, this paper focuses on obtaining information-theoretic bounds on the {average transfer generalization gap}, $\Ebb_{P_{Z^M}P_{W|Z^M}}[\Delta L(W|Z^M)]$, where the average is with respect to the training data and the transfer learner. These bounds are fundamentally different from the existing high-probability bounds, and thus they are not directly comparable. Unlike the high-probability bounds which ignore the properties of the training algorithm, the information-theoretic bounds describe the generalization capability of arbitrary transfer learners via their sensitivity to the input training set.

Our work is related to the recent study in \cite{wu2020information} on information-theoretic bounds for transfer learning. 
 The resulting bound captures the impact of the domain shift via the Kullback-Leibler (KL) divergence $D_{\KL}(P_Z||P'_Z)$ between the source-domain data distribution $P_Z$ and target-domain data distribution $P'_Z$.
 The KL divergence based measure of domain shift suffers from a serious disadvantage: it is well-defined only when the source distribution $P_Z$ is absolutely continuous with respect to $P'_Z$ $(P_Z \ll P'_Z)$, and takes value $\infty$ otherwise. This results in vacuous bounds under various practical conditions, such as for supervised learning problems where the data labels $Y$ are deterministic functions of the feature $X$ within data samples $Z=(X,Y)$; and when the support of the source data distribution includes that of the target data distribution.
\subsection{Contributions}
In this work, we mitigate the above drawback of KL divergence based bounds on average transfer generalization gap, by using a two-parameter $(\alpha_1,\alpha_2)$-family of Jensen-Shannon (JS) divergences with $\alpha_1,\alpha_2 \in [0,1]$ to capture the domain shift. This family includes as special cases the conventional JS divergence with $\alpha_1=\alpha_2=0.5$, as well as Nielsen's symmetric $\alpha$-skew and asymmetric $\alpha$-skew JS divergences \cite{nielsen2020generalization}, which corresponds to the choices $\alpha_2=0.5$ and $\alpha_1=\alpha_2$ respectively. For the setting when data from both source and target distributions are available for training, we obtain new information-theoretic upper bounds on the average transfer generalization gap that capture $(i)$ the impact of the domain shift via the $(\alpha_1,\alpha_2)$-JS divergence between source $P_Z$ and target $P'_Z$ distributions; and $(ii)$ the generalization capability of the transfer learning algorithm through the mutual information between algorithm output and each individual sample of data set. The $(\alpha_1,\alpha_2)$-JS divergence is bounded for $\alpha_1 \in (0,1)$ \cite[Thm.~1]{yamano2019some}, and gives non-vacuous bounds even when $P_Z \not \ll P'_Z$. Moreover, the obtained bound holds for unbounded loss functions with bounded cumulant generating function (CGF). 
 In contrast, the $\phi$-divergence based bound with $\phi(x)=|x-1|$ in \cite[Corollary 3]{wu2020information}, which also holds when $P_Z \not \ll P'_Z$ , requires loss functions to have bounded $L_{\infty}$-norm. 

Our work is motivated by the recent study \cite{shui2020beyond} that employs the conventional JS divergence, with the aim of upper bounding the target domain generalization loss $L_g(w)$ as a function of the source-domain generalization loss for a fixed model parameter $w$. 
Moving beyond \cite{shui2020beyond}, in this work, we consider the performance of a training algorithm that chooses model parameter $w$ by minimizing the weighted average of training losses over source and target data \cite{wu2020information} -- an approach referred to as \textit{empirical weighted risk minimization} (EWRM). We specialize the $(\alpha_1,\alpha_2)$-JS divergence-based bounds on average transfer generalization gap to EWRM, and obtain new upper bounds on the \textit{average optimality gap} for EWRM. This is unlike prior work \cite{ben2010theory,wu2020information}, which obtain high probability bounds on the optimality gap. We show via an example that by choosing the parameters $\alpha_1,\alpha_2$, the $(\alpha_1,\alpha_2)$-JS divergence can better capture the relative impact of source and target data sets on the performance of EWRM, yielding tighter bounds than with the conventional JS divergence.
\section{Problem Formulation}
In transfer learning, we are given a data set that consists of: $(i)$ data points from a \textit{source domain}
with an underlying \textit{unknown} data distribution, $P_{Z} \in \Pscr(\Zscr)$, defined in a subset or vector space $\Zscr$; as well as $(ii)$ data from a \textit{target domain} with a generally different data distribution $P'_{Z} \in \Pscr(\Zscr)$. 
Specifically, 
the learner has  access to a training data set $Z^M=(Z_1,Z_2, \hdots,Z_M)$, which consists of $\beta M$, for some fixed $\beta \in (0,1]$,  independent and identically distributed (i.i.d.) samples $Z^{\beta M}=(Z_1,\hdots, Z_{\beta M}) \sim P^{\beta M}_{Z}$ drawn from the source domain $P_{Z}$, and $(1-\beta)M$ i.i.d. samples $Z^{(1-\beta)M}=(Z_{\beta M+1},\hdots Z_M)\sim P'^{(1-\beta)M}_{Z}$ from the target domain $P'_{Z}$. The learner does not know the distributions $P_{Z}$ and $P'_{Z}$.
  The learner uses the training data set $Z^M$ to choose a model, or hypothesis, $W$ from the model class $\Wscr$ by using a \textit{randomized} learning algorithm defined by a conditional distribution $P_{W|Z^M} \in \Pscr(\Wscr)$ as $W \sim P_{W|Z^M}$.
 The conditional distribution $P_{W|Z^M}$ defines a stochastic mapping from the training data set $Z^M$ to the model class $\Wscr$. 

 The performance of a model parameter vector $w \in \Wscr$ on a data sample $z \in \Zscr$ is measured by a loss function $l(w,z)$ where $l:\Wscr \times \Zscr \rightarrow \Real_{+}$. 
The \textit{generalization loss}, also known as population loss, for a model parameter vector $w \in \Wscr$ is evaluated on the target domain, and is defined as
\begin{align}
&L_{g}(w)=\Ebb_{P'_Z}[l(w,Z)], \label{eq:genloss}
\end{align} where the average is taken over a test example $Z$ drawn independently  of $Z^M$ from the target task data distribution $P'_Z$.
 The generalization loss cannot be computed by the learner, given that the data distribution $P'_Z$ is unknown. A typical solution is for the learner to evaluate instead the \textit{weighted average training loss} on the data set $Z^M$, which is defined as the empirical average  $L_{t}(w|Z^M)=$
\begin{align}
 \frac{\gamma}{\beta M}\sum_{i=1}^{\beta M}l(w,Z_i)+\frac{1-\gamma}{(1-\beta) M}\sum_{i=\beta M+1}^{M}l(w,Z_i), \label{eq:trainingloss}
\end{align}
where $\gamma\in [0,1]$ is a hyperparameter \cite{ben2010theory}, \cite{wu2020information}. We call the algorithm that minimizes \eqref{eq:trainingloss} as the empirical weighted risk minimization (EWRM) algorithm. In formulation, EWRM algorithm outputs
\begin{align}
W^{{\rm EWRM}}(Z^M)= \arg \min_{w \in \Wscr} L_t(w|Z^M)\label{eq:EWRM}
\end{align} for input training set $Z^M$.

The difference between generalization loss \eqref{eq:genloss} and training loss \eqref{eq:trainingloss}, known as \textit{transfer generalization gap}, is defined as
\begin{align}
\Delta L(w|Z^M)=L_g(w)-L_t(w|Z^M), \label{eq:transfergengap}
\end{align}
and is a key metric that relates to the performance of the learner. As mentioned, this is because a small transfer generalization gap ensures that the training loss \eqref{eq:trainingloss} is a reliable estimate of the generalization loss \eqref{eq:genloss}.


\section{$\alpha$-JS Divergence-Based Bounds on Average Transfer Generalization Gap}
In this section, we obtain bounds on the average  transfer generalization gap $\Delta L^{\avg}:=\Ebb_{P_{Z^M}P_{W|Z^M}}[\Delta L(W|Z^M)]$, where the training set distribution is given as $P_{Z^M}=P_Z^{\beta M} \times P'^{(1-\beta)M}_{Z}$. Towards this goal, we assume the following.
\begin{assumption}\label{assum:1}
The loss function $l(W,Z)$ is $\sigma^2$-sub-Gaussian\footnote{A random variable $X \sim P_X$ is said to be $\sigma^2$-sub-Gaussian if its CGF, $\log \Ebb_{P_X}[\exp(\lambda(X-\Ebb_{P_X}[X]))]$, is upper bounded by $\lambda^2 \sigma^2/2$ for all $\lambda \in \Real$.} under $(W,Z) \sim P_W R^{\alpha_1}_Z$, where $P_W$ is the marginal of the joint distribution $P_{W|Z^M} P_{Z^M}$ and \begin{align}
R_Z^{\alpha_1}(z)=\alpha_1 P_Z(z)+(1-\alpha_1)P'_Z(z), \label{eq:mixturedistribution}
\end{align} for some $\alpha_1 \in [0,1]$, is a mixture of the source and target data distributions.
\end{assumption}

Note that if the loss function is bounded, \ie, $0\leq a \leq l(\cdot,\cdot) \leq b< \infty$, Assumption~\ref{assum:1} is satisfied with $\sigma^2=(b-a)^2/4$ under any data distribution $R^{\alpha_1}_Z$ for $\alpha_1 \in [0,1]$.

To derive bounds on the average transfer generalization gap, we consider the following family of $(\alpha_1,\alpha_2)$-JS divergences,
\begin{align}
D_{\JS}^{\alpha_1,\alpha_2}(P'_Z||P_Z)&=\alpha_2 D_{\KL}(P'_Z||R_Z^{\alpha_1})\non \\&+(1-\alpha_2)D_{\KL}(P_Z||R_Z^{\alpha_1})), \label{eq:alpha1alpha2JSdivergence}
\end{align}
where $\alpha_1, \alpha_2 \in [0,1]$. We refer to Section~\ref{sec:intr} for connections with existing JS divergences.

  Towards obtaining $(\alpha_1,\alpha_2)$-JS-divergence-based bounds, we decompose the transfer generalization gap \eqref{eq:transfergengap} as
\begin{align}
\Delta L(w|Z^M)&=\gamma (L_g(w)-L_t(w|Z^{\beta M}))\non \\&+(1-\gamma)(L_g(w)-L_t(w|Z^{(1-\beta) M})), \label{eq:decomposition} 
\end{align}
where $L_t(w|Z^{\beta M})=\sum_{i=1}^{\beta M}l(w,Z_i)/(\beta M)$ is the training loss over the source-domain data and $L_t(w|Z^{(1-\beta) M})=\sum_{i=\beta M+1}^{ M}l(w,Z_i)/((1-\beta) M)$ is the training loss of the target-domain data. By separately bounding the average of the two differences in the above decomposition, we obtain the following bound.
\begin{theorem}\label{thm:averagebound}
Under Assumption~\ref{assum:1} and for $(\beta,\alpha_2) \in (0,1)$, the following upper bound on the average transfer generalization gap holds for any algorithm $P_{W|Z^M}$,
\begin{align}
&\Delta L^{\avg}\hspace{-0.1cm} \leq \hspace{-0.1cm} \frac{\gamma\sigma\sqrt{2 \hat{\alpha_2}}}{\beta M}\sum_{i=1}^{\beta M} \hspace{-0.1cm}\sqrt{ D_{\JS}^{\alpha_1,\alpha_2}(P'_Z||P_Z)+(1-\alpha_2)I(W;Z_i)}\non \\
&\hspace{-0.2cm}+\frac{2(1-\gamma)\sigma}{(1-\beta)M}\sum_{i=\beta M+1}^M \sqrt{2 D_{\KL}(P'_Z||R_Z^{\alpha_1})+I(W;Z_i)}, \label{eq:avgbound}
\end{align}
where $\hat\alpha_2=1/\alpha_2 +1/(1-\alpha_2)$.
\end{theorem}
\begin{proof}
See Appendix~\ref{app:averagebound}.
\end{proof}
\vspace{-0.05cm}
The first term in \eqref{eq:avgbound} accounts for the contribution to the transfer generalization gap caused by the limited availability of the source-domain data. It comprises of $(i)$ the sensitivity measure  of the algorithm to the individual sample of the source-domain training set captured by the mutual information $I(W;Z_i)$; and $(ii)$ the domain shift between source and target data distributions captured by the $(\alpha_1,\alpha_2)$-JS-divergence $D^{\alpha_1,\alpha_2}_{\JS}(P'_Z||P_Z)$.
The second term of \eqref{eq:avgbound} similarly accounts for the contribution of the limited data from the target-domain. It comprises of the mutual information $I(W;Z_i)$ which accounts for the sensitivity of the learning algorithm to individual sample of the target-domain training set; and of  the KL divergence term $D_{\KL}(P'_Z||R^{\alpha_1}_Z)$, which quantify the distance between the target distribution $P'_Z$ and the mixture distribution $R^{\alpha_1}_Z$. 

We note that the KL divergence term $D_{\KL}(P'_Z||R^{\alpha_1}_Z)$ arises here since the sub-Gaussianity of the loss function $l(W,Z)$ is assumed under $(W,Z) \sim P_WR^{\alpha_1}_Z$ (Assumption~\ref{assum:1}). We also note that, for $\alpha_1<1$, we have $\supp(P'_Z) \subseteq \supp(R^{\alpha_1}_Z)$, and hence the KL divergence $D_{\KL}(P'_Z||R_Z^{\alpha_1})$ is well-defined.  Moreover, for fixed $\gamma$, $\beta$ and $M$, the upper bound in \eqref{eq:avgbound} can be tightened by optimizing over the choice of $\alpha_1$ and $\alpha_2$.  For instance, for the extreme case when $\gamma=0$, the bound in \eqref{eq:avgbound} is minimized by choosing $\alpha_1=\gamma=0$. 

Note that the bound in \eqref{eq:avgbound} does not account for the case $\beta=0$, \ie, when only target-domain data set is available for training. In this case, the problem reduces to the conventional learning with $P_Z=P'_Z$.
We now specialize the bound in \eqref{eq:avgbound} to the case when only data from source distribution is available for training, \ie, when $\beta=1$. 
\begin{corollary}
Under Assumption~\ref{assum:1}, the following bound holds when $\beta=1$, $\Delta L^{\avg} \leq $
\begin{align}
&\frac{\sigma\sqrt{2 \hat{\alpha_2}}}{M}\sum_{i=1}^{M} \sqrt{ D_{\JS}^{\alpha_1,\alpha_2}(P'_Z||P_Z)+(1-\alpha_2)I(W;Z_i)}.\label{eq:avgbound_onlysource}
\end{align}
\end{corollary}
\vspace{-0.1cm}
The bound in \eqref{eq:avgbound} can be proven to hold also under the following assumption, similar to the one considered in \cite{xu2017information}.
\begin{assumption}\label{assum:2}
The loss function $l(w,Z)$ is $\sigma^2-$sub-Gaussian under $Z \sim R_Z^{\alpha_1}$ for all $w \in \Wscr$.
\end{assumption}

To see this, one can follow the steps in the derivation of the exponential inequalities in Lemma~\ref{lem:expinequality_avg} of Appendix~\ref{app:averagebound}, starting from the additional step of averaging both sides of the inequality $\Ebb_{R_Z^{\alpha_1}}[\exp(\lambda(l(w,Z)-\Ebb_{R_Z^{\alpha_1}}[l(w,Z)])-\lambda^2\sigma^2/2)]\leq 1$ over $W \sim P_W$. As discussed in \cite{bu2019tightening}, in general, Assumption~\ref{assum:1} does not imply this assumption, and vice versa. However, both assumptions hold when $l(\cdot,\cdot)$ is bounded. 

We finally note that the $(\alpha_1,\alpha_2)$-JS-divergence-based bounds on average transfer generalization gap can be generalized to loss functions $l(W,Z)$ whose CGF is upper bounded by a function $\Psi(\lambda)$ for $\lambda \in [b_{-},b_{+}]$ under $(W,Z)\sim P_W R^{\alpha_1}_Z$. See Appendix~\ref{app:boundedCGF} for details.
This class of functions include the sub-Gaussian loss function $l(W,Z)$ in Assumption~\ref{assum:1} with $\Psi(\lambda)=\Psi(-\lambda)=\lambda^2 \sigma^2/2$ and $b_{+}=b_{-}=\infty$, and the sub-gamma loss $l(W,Z)$ with variance parameter $\sigma$ and scale parameter $c$, whose CGF is upper bounded by $\Psi(\lambda)= \lambda^2 \sigma^2/2(1-c|\lambda|)$ for $|\lambda|<1/c$.

\subsection{Bound on Average Transfer Excess  Risk for EWRM}
In this section, we obtain an upper bound on the average transfer excess risk of  EWRM. Let
\begin{align}
w^{*}=\arg \min_{w \in \Wscr} L_g(w)
\end{align} be the optimizing model parameter of the transfer generalization loss $L_g(w)$. Then, the average transfer excess risk for the EWRM algorithm is defined as
\begin{align}
\Delta L_g^{*}=\Ebb_{P_{Z^M}}[L_g(W^{\EWRM})]-L_g(w^{*}),
\end{align}where we have used $W^{\EWRM}$ to denote $W^{\EWRM}(Z^M)$ for notational convenience.

To obtain an upper bound on the  average excess risk $\Delta L_g^{*}$, we use the decomposition
\begin{align}
&\Delta L_g^{*}=\underbrace{\Ebb_{P_{Z^M}}[L_g(W^{\EWRM})]-\Ebb_{P_{Z^M}}[L_t(W^{\EWRM}|Z^M)]}_{A} \non \\&+ \underbrace{\Ebb_{P_{Z^M}}[L_t(W^{\EWRM}|Z^M)]-L_g(w^{*})}_{B}. \label{eq:excessrisk_defn}
\end{align} Term A in \eqref{eq:excessrisk_defn} corresponds to the average transfer generalization gap for the EWRM, and hence it can be upper bounded using \eqref{eq:avgbound}. Using the definition \eqref{eq:EWRM} of EWRM, term B can be upper bounded as
\begin{align}
B &\leq \Ebb_{P_{Z^M}}[L_t(w^{*}|Z^M)]-L_g(w^{*})\non \\ 
&=\gamma \biggl[ \Ebb_{P_Z}[l(w^{*},Z)]-\Ebb_{P'_Z}[l(w^{*},Z)]\biggr],
\end{align}
where the last equality follows from \eqref{eq:trainingloss} and using  the identity $\Ebb_{P_{Z^{(1-\beta) M}}}[L_t(w^{*}|Z^{(1-\beta) M})]=L_g(w^{*})$.
Denoting the upper bound on term A which follows from  \eqref{eq:avgbound} as ${\rm{UB}}(W^{\EWRM})=$
\begin{align*}
&\frac{\gamma\sigma\sqrt{2 \hat{\alpha_2}}}{\beta M}\sum_{i=1}^{\beta M} \sqrt{ D_{\JS}^{\alpha_1,\alpha_2}(P'_Z||P_Z)+(1-\alpha_2)I(W^{\EWRM};Z_i)}\non \\
&\hspace{-0.2cm}+\frac{2(1-\gamma)\sigma}{(1-\beta)M}\sum_{i=\beta M+1}^M \sqrt{2 D_{\KL}(P'_Z||R_Z^{\alpha_1})+I(W^{\EWRM};Z_i)},
\end{align*}
and combining this with an upper bound on term B yields the following bound on the average transfer excess risk for EWRM.
%
%
\begin{theorem}\label{thm:excessrisk_JS}
Under Assumption~\ref{assum:2}, the following bound holds for $\beta \in (0,1]$
\begin{align} 
&\Delta L_g^{*}\leq  {\rm{UB}}(W^{\EWRM})+\gamma \sqrt{2\sigma^2\hat{\alpha_2}D^{\alpha_1,\alpha_2}_{\JS}(P'_Z||P_Z)}\label{eq:excessrisk_JS},
\end{align}
where $\hat\alpha_2=1/\alpha_2 +1/(1-\alpha_2)$.
\end{theorem}
\begin{proof}
See Appendix~\ref{app:excessrisk_JS}.
\end{proof}
\section{Example}
In this section, we consider the problem of estimating the mean of a discrete random variable $Z$ taking values in set $\Zscr=\{0,1,2\}$. The source domain is defined by data distributed as $Z\sim P_Z$, with $P_Z(0)=p_s$ and $P_Z(1)=1-p_s$, and the target-domain data is distributed as $Z \sim P'_Z$, with $P'_Z(1)=p_t$ and $P'_Z(2)=1-p_t$.
The transfer learner infers an estimate $w \in \Wscr$ of the mean of the random variable $Z$. The loss function $l(w,z)=(w-z)^2$ measures the quadratic error between the estimate $w$ and a test input $z$.
For a training data set $Z^M$, the EWRM transfer learner in \eqref{eq:EWRM} outputs the estimate 
\begin{align}
W^{\EWRM}= \frac{\gamma}{\beta M} \sum_{i=1}^{\beta M}Z_i + \frac{(1-\gamma)}{(1-\beta) M} \sum_{i=\beta M+1}^{ M}Z_i. \label{eq:EWRM_example}
\end{align}

The average transfer generalization gap evaluates to
\begin{align}
&\Ebb_{P_{Z^M}}[\Delta L(W^{\EWRM}|Z^M)]=2\Ebb_{P_{Z^M}}[(W^{\EWRM})^2]\non \\&-2\mu_t\Ebb_{P_{Z^M}}[W^{\EWRM}]+\gamma(\nu_t+\mu_t^2-\nu_s-\mu_s^2), \label{eq:gap_example}
\end{align} 
where $\mu_t$ and $\nu_t$ are the mean and variance respectively of the random variable $Z \sim P'_Z$; while $\mu_s$, and $\nu_s$ are the mean and variance respectively of the random variable $Z \sim P_Z$. The averages in \eqref{eq:gap_example} can be computed explicitly as $\Ebb_{P_{Z^M}}[W^{\EWRM}]=\gamma \mu_s+(1-\gamma)\mu_t$, and $\Ebb_{P_{Z^M}}[(W^{\EWRM})^2]=\gamma^2 \nu_s/(\beta M)+(1-\gamma)^2\nu_t/((1-\beta)M)+(\Ebb_{P_{Z^M}}[W^{\EWRM}])^2$.

Since the support of the target-domain data distribution does not include the support of the source-domain data distribution, the KL divergence evaluates to $D(P_Z||P'_Z)=\infty$. In contrast, the $(\alpha_1,\alpha_2)$-JS divergence can be evaluated in closed form.
 Furthermore, using \eqref{eq:EWRM_example} and the alphabet $\Zscr \in \{0,1,2\}$, we can, without loss of generality, consider the model parameter space $\Wscr$ limited to the interval $[0,2]$. Therefore, the loss function $l(w,z)$ is bounded in the interval $[0,4]$, and hence it is $4$-sub-Gaussian.
\begin{figure}[h!]
 \centering 
   \includegraphics[scale=0.5,trim=3.2in 1.2in 2.5in 1.5in,clip=true]{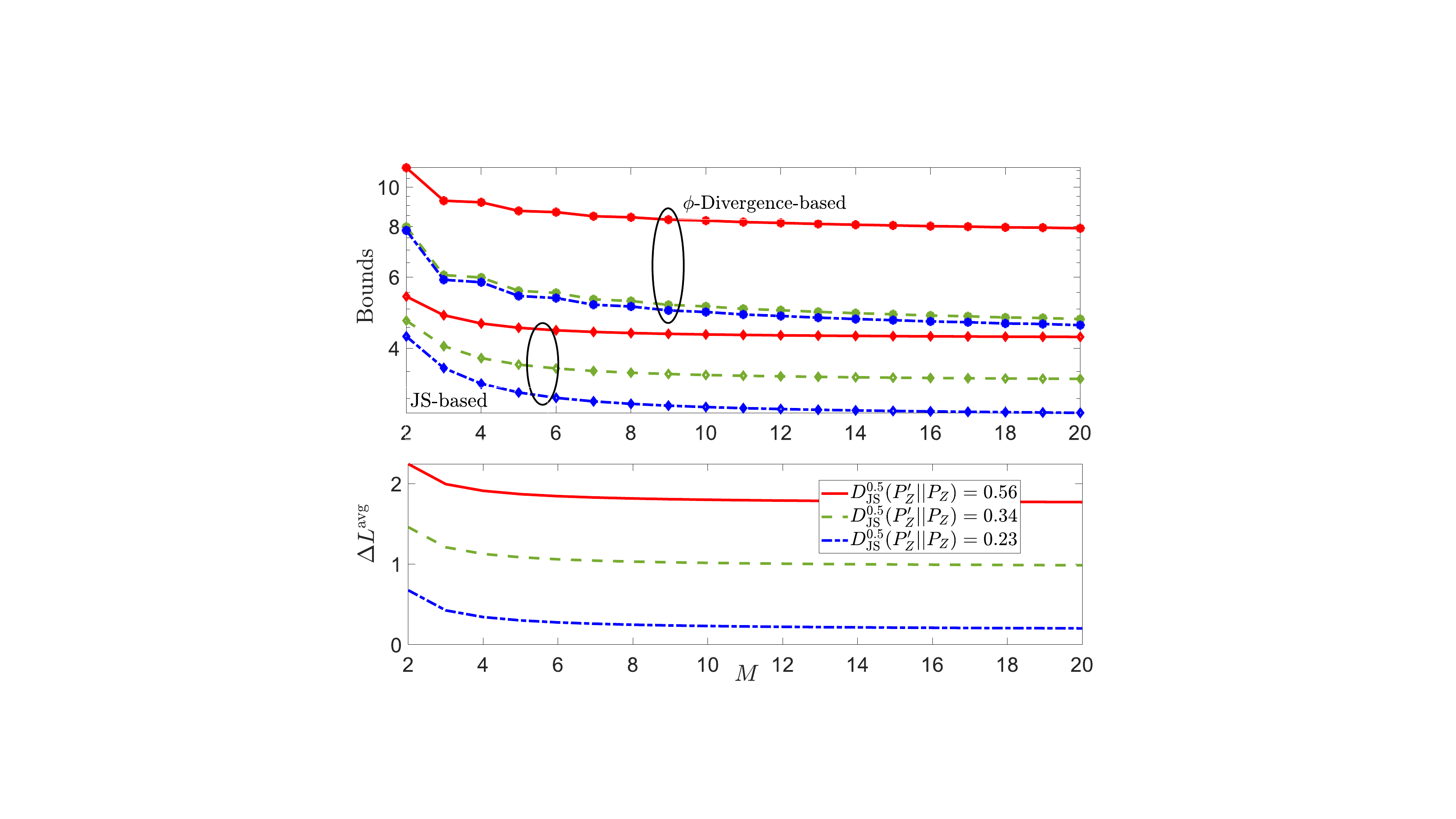} 
  \caption{Average transfer generalization gap \eqref{eq:gap_example} (bottom) and the $(0.5,0.5)$-JS-based bound in \eqref{eq:avgbound_onlysource}  and $\phi$-divergence based bound in \cite[Cor.~3]{wu2020information} (top) as a function of $M$ (when $\beta=1$) for varying JS divergence between $P'_Z$ and a fixed $P_Z$ with $p_s=0.48$.} \label{fig:exp3_fig1}
  \vspace{-0.3cm}
   \end{figure}

In Figure~\ref{fig:exp3_fig1}, we compare the the average transfer generalization gap \eqref{eq:gap_example} with the  conventional JS-divergence bound of \eqref{eq:avgbound_onlysource} for $\alpha_1=\alpha_2=0.5$ and the $\phi$-divergence based bound in \cite[Cor.~3]{wu2020information} with $\phi(x)=|x-1|$,  for the case when $\beta=1$ (\ie, only source-domain data set available for training) as a function of  increasing values of $M$. For fixed $P_Z$ with $p_s=0.48$, we vary the JS-divergence by varying $p_t$. As predicted by our bound, the transfer generalization gap decreases with increase in the number of source-data samples $M$ available for training. However, there exists a non-vanishing generalization gap even at high $M$, which is a direct consequence of the domain shift. Moreover, a larger JS-divergence between $P_Z$ and $P'_Z$ is predictive of a larger average transfer generalization gap. Finally, we show that JS-divergence based bounds outperform the  $\phi$-divergence based bound in \cite[Cor.~3]{wu2020information} when $\beta=1$ at varying JS distances. 

 We now study the advantage of considering the general family of $(\alpha_1,\alpha_2)$-JS divergence over the JS divergence. Since the loss function is bounded, Assumption~\ref{assum:1} holds for mixture distribution $R^{\alpha_1}_Z$ for any $\alpha_1 \in [0,1]$. Consequently, the bound in \eqref{eq:avgbound} can be tightened by optimizing over $\alpha_1$ and $\alpha_2$.
\begin{figure}[h!]
 \centering 
   \includegraphics[scale=0.48,trim=3.1in 1.55in 2.5in 1.85in,clip=true]{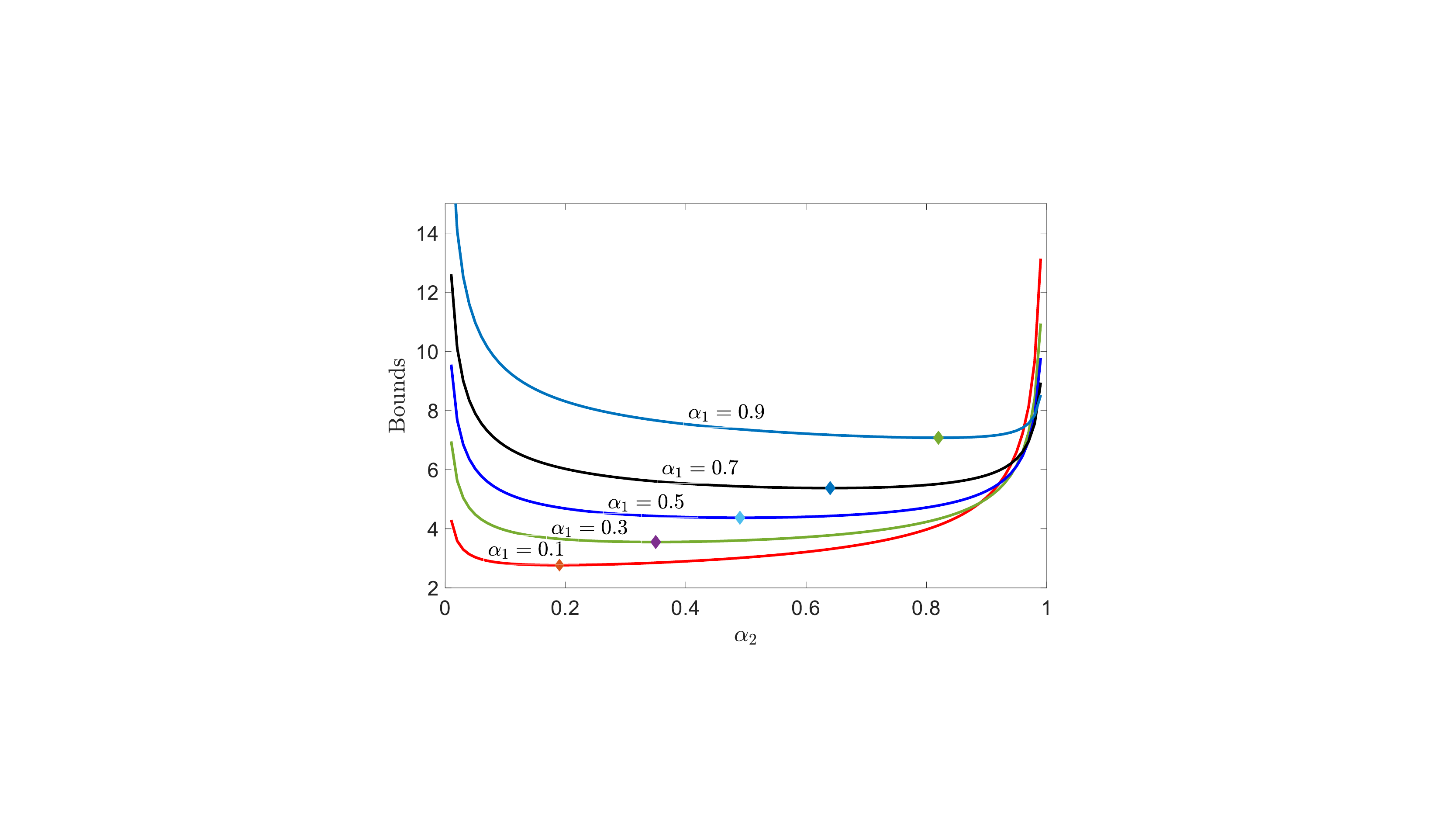} 
   \caption{The $(\alpha_1,\alpha_2)$-JS divergence based bound in \eqref{eq:avgbound} as a function of $\alpha_2$ for varying $\alpha_1$ when $M=30$, $\beta=2/3$ and $\gamma=0.3$. } \label{fig:exp5_fig1}
   \vspace{-0.3cm}
  \end{figure}  
In Figure~\ref{fig:exp5_fig1}, we evaluate the tightness of the bound in \eqref{eq:avgbound} as a function of $\alpha_2$ for varying values of $\alpha_1$. As can be seen, the choice $(\alpha_1=0.1$, $\alpha_2 \approx 0.2)$  yields the tightest bound. Therefore, the optimizing choice of $(\alpha_1,\alpha_2)$ does not result in existing JS divergences which assume $\alpha_2=0.5$ (symmetric skew JS divergence) or $\alpha_2=\alpha_1$ (asymmetric skew JS divergence).
   \appendices
\section{Proof of Theorem~\ref{thm:averagebound} }\label{app:averagebound}
To obtain an upper bound on $\Delta L^{\avg}$, we use the decomposition \eqref{eq:decomposition} and separately bound the two differences. The average of the first difference in \eqref{eq:decomposition} can be equivalently written as $\Ebb_{P_{Z^M,W}}[L_g(W)-L_t(W|Z^{\beta M})]=$
\begin{align}
&\frac{1}{\beta M}\sum_{i=1}^{\beta M}\biggl[\Ebb_{P_W P'_{Z_i}}[l(W,Z_i)]-\Ebb_{P_{Z_i}P_{W|Z_i}}[l(W,Z_i)] \biggr] \label{eq:firstdiff}
\end{align} 
and similarly $\Ebb_{P_{Z^M,W}}[L_g(W)-L_t(W|Z^{(1-\beta) M})]=$
\begin{align}
&\frac{\sum_{i=\beta M+1}^{M}\biggl[\Ebb_{P_W P'_{Z_i}}[l(W,Z_i)]-\Ebb_{P'_{Z_i}P_{W|Z_i}}[l(W,Z_i)] \biggr]}{(1-\beta) M} \label{eq:secdiff}.
\end{align}
We first bound the difference $\Ebb_{P_W P'_{Z_i}}[l(W,Z_i)]-\Ebb_{P_{Z_i}P_{W|Z_i}}[l(W,Z_i)]$ in \eqref{eq:firstdiff}. Towards this, we use the exponential inequalities in Lemma~\ref{lem:expinequality_avg} obtained based on the change of measure approach adopted in \cite{hellstrom2020generalization}. Fix $\lambda=\lambda_1/\alpha_2$ for some $\lambda_1>0$ in \eqref{eq:expinequality-1} and $\lambda=-\lambda_1/(1-\alpha_2)$ in \eqref{eq:expinequality-2}, and apply Jensen's inequality yield the following inequalities
\begin{align}
&\Ebb_{P_W P'_{Z_i}}[l(W,Z_i)]-\Ebb_{P_W R^{\alpha_1}_Z}[l(W,Z)]\leq  \frac{\lambda_1 \sigma^2}{2\alpha_2}\non \\& \quad+\frac{\alpha_2 D_{\KL}(P'_Z||R^{\alpha_1}_Z)}{\lambda_1}  \label{ineq:1}
\\
&\Ebb_{P_W R_Z^{\alpha_1}}[l(W,Z)]-\Ebb_{P_{Z_i}P_{W|Z_i}}[l(W,Z_i)] \leq \frac{\lambda_1 \sigma^2}{2 (1-\alpha_2)} \non \\
 &\hspace{1cm}+\frac{1-\alpha_2}{\lambda_1}\biggl(D_{\KL}(P_Z||R^{\alpha_1}_Z)+I(W;Z_i) \biggr)
 .\label{ineq:2}
\end{align}
Adding \eqref{ineq:1} and \eqref{ineq:2} and choosing $$\lambda_1=\sqrt{\frac{D_{\JS}^{\alpha_1,\alpha_2}(P'_Z||P_Z)+(1-\alpha_1)I(W;Z_i)}{\frac{\sigma^2}{2}\Bigl(\frac{1}{\alpha_2}+\frac{1}{1-\alpha_2}\Bigr)}}$$ gives that $\Ebb_{P_W P'_{Z_i}}[l(W,Z_i)]- \Ebb_{P_{Z_i}P_{W|Z_i}}[l(W,Z_i)] \leq$
$$ \sqrt{2\sigma^2\biggl(\frac{1}{\alpha_2}+\frac{1}{1-\alpha_2}\biggr)\biggl( D_{\JS}^{\alpha_1,\alpha_2}(P'_Z||P_Z)+(1-\alpha_2)I(W;Z_i)\biggr)}.
$$
Similarly, we can bound the difference $\Ebb_{P_W P'_{Z_i}}[l(W,Z_i)]-\Ebb_{P'_{Z_i}P_{W|Z_i}}[l(W,Z_i)]$ in \eqref{eq:secdiff} by fixing $ \lambda=\lambda_1 >0$ in \eqref{eq:expinequality-1} and $\lambda=-\lambda_1$ in \eqref{eq:expinequality-3}. Applying Jensen's inequality on both bounds, adding the resultant inequalities  and choosing $\lambda_1=\sqrt{2D_{\KL}(P'_Z||R^{\alpha_1}_Z)+I(W;Z_i)}/\sigma$ gives the corresponding bound.
\begin{lemma}\label{lem:expinequality_avg}
Under Assumption~\ref{assum:1}, the following inequalities hold for all $\lambda \in \Real$ when $i=1,\hdots,\beta M$,

\begin{align}
&\Ebb_{P_W P'_{Z_i}}\biggl[ \exp \biggl(\lambda(l(W,Z_i)-\Ebb_{P_W R_Z^{\alpha_1}}[l(W,Z)])-\frac{\lambda^2 \sigma^2}{2} \non \\& \qquad -\log \frac{P'_{Z_i}(Z_i)}{R_{Z_i}^{\alpha_1}(Z_i)}\biggr)\biggr] \leq 1, \label{eq:expinequality-1} \\
&\Ebb_{P_{Z_i}P_{W|Z_i}}\biggl[ \exp \biggl(\lambda(l(W,Z_i)-\Ebb_{P_W R_Z^{\alpha_1}}[l(W,Z)])-\frac{\lambda^2 \sigma^2}{2} \non \\& \qquad -\log \frac{P_{Z_i}(Z_i)}{R_{Z_i}^{\alpha_1}(Z_i)}-\imath(W,Z_i)\biggr)\biggr] \leq 1, \label{eq:expinequality-2}
\end{align}
where $\imath(W,Z_i)=\log (P_{W,Z_i}(W,Z_i)/(P_WP_{Z_i}(W,Z_i))$ is the information density between random variables $W$ and $Z_i$.
For $i=\beta M+1,\hdots,M$, the inequality \eqref{eq:expinequality-1} holds along with the following inequality
\begin{align}
&\Ebb_{P'_{Z_i}P_{W|Z_i}}\biggl[ \exp \biggl(\lambda(l(W,Z_i)-\Ebb_{P_W R_Z^{\alpha_1}}[l(W,Z)])-\frac{\lambda^2 \sigma^2}{2} \non \\& \qquad -\log \frac{P'_{Z_i}(Z_i)}{R_{Z_i}^{\alpha_1}(Z_i)}-\imath(W,Z_i)\biggr)\biggr] \leq 1. \label{eq:expinequality-3}
\end{align}
\end{lemma}
\begin{proof}
See Appendix~\ref{app:proof_Lemma}.
\end{proof}
\section{Proof of Lemma~\ref{lem:expinequality_avg}}\label{app:proof_Lemma}
The derivation of the exponential inequalities leverage the change of measure approach adopted in \cite{hellstrom2020generalization}. 
For $i=1,\hdots, M$, Assumption~\ref{assum:1} gives that
\begin{align}
\Ebb_{P_W R^{\alpha_1}_{Z_i}}\biggl[ \exp \biggl(\lambda(l(W,Z_i)-\Ebb_{P_W R_Z^{\alpha_1}}[l(W,Z)])-\frac{\lambda^2 \sigma^2}{2} \biggr)\biggr] \leq 1. \label{eq:subGaussian-1}
\end{align}
For $i=1,\hdots, \beta M$, the inequality \eqref{eq:subGaussian-1} implies the following inequality,
\begin{align}
\Ebb_{P_W R^{\alpha_1}_{Z_i}}\biggl[\Ibb_{\Escr} & \exp \biggl(\lambda(l(W,Z_i)-\Ebb_{P_W R_Z^{\alpha_1}}[l(W,Z)])-\frac{\lambda^2 \sigma^2}{2} \biggr)\biggr] \non \\&\leq 1, \label{eq:subGaussian-11}
\end{align}
where $\Escr= \supp(P'_{Z_i})$, and $\Ibb_{E}$ denotes the indicator function which takes value $1$ when the event $E$ is true, and is zero otherwise. Now, performing a change of measure from $R_{Z_i}^{\alpha_1}$ to $P'_{Z_i}$ as in \cite[Prop. 17]{polyanskiy2014lecture} yields \eqref{eq:expinequality-1}.

 To get to inequality \eqref{eq:expinequality-2}, we similarly first perform change of measure from $R^{\alpha_1}_{Z_i}$ to $P_{Z_i}$ on \eqref{eq:subGaussian-1} to get the inequality
 \begin{align}
 &\Ebb_{P_W P_{Z_i}}\biggl[ \exp \biggl(\lambda(l(W,Z_i)-\Ebb_{P_W R_Z^{\alpha_1}}[l(W,Z)])-\frac{\lambda^2 \sigma^2}{2} \non \\& \qquad -\log \frac{P_{Z_i}(Z_i)}{R_{Z_i}^{\alpha_1}(Z_i)}\biggr)\biggr] \leq 1. \label{eq:expinequality-12}
 \end{align}
 The inequality \eqref{eq:expinequality-12} implies the following
 \begin{align}
 &\Ebb_{P_W P_{Z_i}}\biggl[\Ibb_{\Escr_1} \exp \biggl(\lambda(l(W,Z_i)-\Ebb_{P_W R_Z^{\alpha_1}}[l(W,Z)])-\frac{\lambda^2 \sigma^2}{2} \non \\& \qquad -\log \frac{P_{Z_i}(Z_i)}{R_{Z_i}^{\alpha_1}(Z_i)}\biggr)\biggr] \leq 1, \label{eq:expinequality-123}
 \end{align} where $\Escr_1=\supp(P_{W,Z_i})$. Now performing a change of measure from
 $P_{Z_i}P_W$ to $P_{Z_i}P_{W|Z_i}$ as in \cite[Prop. 17]{polyanskiy2014lecture}, and using the definition of information density $\imath(W,Z_i)$ yields \eqref{eq:expinequality-2}.

For $i=\beta M+1,\hdots,M$, \eqref{eq:expinequality-1} can be verified to hold as before. To get to inequality \eqref{eq:expinequality-3}, we perform a change of measure on \eqref{eq:expinequality-1} from $ P'_{Z_i}P_W$ to $P'_{Z_i}P_{W|Z_i}$ as explained before.
\section{Proof of Theorem~\ref{thm:excessrisk_JS}}\label{app:excessrisk_JS}
To obtain the required bound in \eqref{eq:excessrisk_JS}, we show that the term $ \Ebb_{P_Z}[l(w^{*},Z)]-\Ebb_{P'_Z}[l(w^{*},Z)]$ can be bounded by $\sqrt{2\sigma^2D^{\alpha_1,\alpha_2}_{\JS}(P'_Z||P_Z)\hat{\alpha}_2}$, where $\hat{\alpha}_2=1/\alpha_2+1/(1-\alpha_2)$. Towards this, note that the following exponential inequality holds under Assumption~\ref{assum:2} with $w=w^{*}$
\begin{align}
\Ebb_{R_Z^{\alpha_1}}\biggl[\exp\biggl(\lambda(l(w^{*},Z)-\Ebb_{R_Z^{\alpha_1}}[l(w^{*},Z)]) -\frac{\lambda^2\sigma^2}{2}\biggr)\biggr] \leq 1. \label{eq:excess_1}
\end{align}
Denote $\delta l(w^{*}):=l(w^{*},Z)-\Ebb_{R_Z^{\alpha_1}}[l(w^{*},Z]$. Now, performing change of measure from $R_Z^{\alpha_1}$ to $P'_Z$, and from $R^{\alpha_1}_Z$ to $P_Z$ on \eqref{eq:excess_1} respectively as in Lemma~\ref{lem:expinequality_avg}, yields the following inequalities
\begin{align}
&\Ebb_{P'_Z}\biggl[\exp\biggl(\lambda\delta l(w^{*})-\log \frac{P'_Z (Z)}{R_Z^{\alpha_1}(Z)}-\frac{\lambda^2\sigma^2}{2}\biggr)\biggr] \leq 1, \label{eq:excess_2}\\
&\Ebb_{P_Z}\biggl[\exp\biggl(\lambda\delta l(w^{*}) -\log \frac{P_Z(Z)}{R_Z^{\alpha_1}(Z)}-\frac{\lambda^2\sigma^2}{2}\biggr)\biggr] \leq 1. \label{eq:excess_3}
\end{align}Take $\lambda=\lambda_1/\alpha_2$ for some $\lambda_1>0$ in \eqref{eq:excess_3} and $\lambda=-\lambda_1/(1-\alpha_2)$ in \eqref{eq:excess_2}. Now apply Jensen's inequality on \eqref{eq:excess_2} and \eqref{eq:excess_3}, and add the resulting inequalities to get the following inequality
\begin{align}
\Ebb_{P_Z}[l(w^{*},Z)]-L_g(w^{*}) \leq \frac{\lambda_1 \sigma^2 \hat{\alpha}_2}{2}+\frac{1}{\lambda_1}D_{\JS}^{\alpha_1,\alpha_2}(P'_Z||P_Z).
\end{align}Now, letting $\lambda_1=\sqrt{D_{\JS}^{\alpha_1,\alpha_2}(P'_Z||P_Z)/(0.5\sigma^2 \hat{\alpha}_2)}$ yields the required bound.
\section{$(\alpha_1,\alpha_2)$-JS Divergence Based Bound for Loss Functions with Bounded CGF}\label{app:boundedCGF}
\begin{theorem}\label{lem:general}
Assume that the CGF of the loss function $l(W,Z)$ is upper bounded by a function $\Psi(\lambda)$ when $(W,Z) \sim P_W R^{\alpha}_Z$ for $\lambda \in [b_{-},b_{+}]$. Then, for $\beta \in (0,1)$ and $\alpha_2 \in (0,1)$, the following upper bound on the average transfer generalization gap holds,
\begin{align}
&\Delta L^{\avg}\leq  \frac{\gamma}{\beta M}\sum_{i=1}^{\beta M} \inf_{\lambda_1 \in (0,b)}\biggl(\frac{1}{\lambda_1}\biggl(\widehat{\Psi}(\lambda_1)  + D_{\JS}^{\alpha_1,\alpha_2}(P'_Z||P_Z)\non \\& +(1-\alpha_2)I(W;Z_i) \biggr) +\frac{1-\gamma}{(1-\beta)M}\sum_{i=\beta M+1}^M  \inf_{\lambda_2 \in (0,b_2)}\biggl(\frac{1}{\lambda_2}\biggl(\Psi(\lambda_2)\non \\&+\Psi(-\lambda_2)+2D_{\KL}(P'_Z||R_Z^{\alpha_1})+I(W;Z_i) \biggr)\biggr), \label{eq:avgbound_generalloss}
\end{align}
where $$\widehat{\Psi}(\lambda)=\alpha_2\Psi\Bigl(\frac{\lambda_1}{\alpha_2}\Bigr)+(1-\alpha_2)\Psi\Bigl(\frac{-\lambda_1}{1-\alpha_2}\Bigr),$$ $b=\min\{\alpha_2 b_{+},-(1-\alpha_2)b_{-}\}$ and $b_2=\min \{b_{+},-b_{-}\}$.
\end{theorem}
\begin{proof}
The assumption that the CGF of the loss function is bounded by $\Psi(\lambda)$ implies the following inequality for $i=1,\hdots,M$
\begin{align}
\Ebb_{P_W R^{\alpha}_{Z_i}}\biggl[ \exp \biggl(\lambda(l(W,Z_i)-\Ebb_{P_W R_Z^{\alpha}}[l(W,Z)])-\Psi(\lambda) \biggr)\biggr] \leq 1, \label{eq:general-1}
\end{align}for all $\lambda \in [b_{-},b_{+}]$. It is then easy to see that 
the exponential inequalities in Lemma~\ref{lem:expinequality_avg} hold with the term $\lambda^2 \sigma^2/2$ replaced by $\Psi(\lambda)$.

Now, proceed as in the proof of Theorem~\ref{thm:averagebound} in Appendix~\ref{app:averagebound} to bound the difference $\Ebb_{P_W P'_Z}[l(W,Z)]-\Ebb_{P_{Z_i}P_{W|Z_i}}[l(W,Z_i)]$. This results in the following inequalities replacing \eqref{ineq:1} and \eqref{ineq:2}
\begin{align}
&\Ebb_{P_W P'_{Z_i}}[l(W,Z_i)]-\Ebb_{P_W R^{\alpha_1}_Z}[l(W,Z)]\non \\&\leq  \frac{\alpha_2}{\lambda_1}\biggl(\Psi\Bigl(\frac{\lambda_1}{\alpha_2}\Bigr)+ D_{\KL}(P'_Z||R^{\alpha_1}_Z)\biggr) \label{ineq:11}
\\
&\Ebb_{P_W R_Z^{\alpha_1}}[l(W,Z)]-\Ebb_{P_{Z_i}P_{W|Z_i}}[l(W,Z_i)] \non \\&\leq \frac{1-\alpha_2}{\lambda_1}\biggl(\Psi\Bigl(\frac{-\lambda_1}{1-\alpha_2}\Bigr)+ D_{\KL}(P'_Z||R^{\alpha_1}_Z)+I(W;Z_i)\biggr) 
 ,\label{ineq:22}
\end{align} which holds for all $\lambda_1 \in (0,\min\{\alpha_2 b_{+}, -(1-\alpha_2)b_{-})$.
Adding \eqref{ineq:11} and \eqref{ineq:22}, and subsequently optimizing over $\lambda_1$ yields the required bound on $\Ebb_{P_W P'_Z}[l(W,Z)]-\Ebb_{P_{Z_i}P_{W|Z_i}}[l(W,Z_i)]$.

To bound the difference $\Ebb_{P_W P'_Z}[l(W,Z)]-\Ebb_{P'_{Z_i}P_{W|Z_i}}[l(W,Z_i)]$, we proceed as in Appendix~\ref{app:averagebound}  by fixing $ \lambda=\lambda_2$ in \eqref{eq:expinequality-1} and $\lambda=-\lambda_2$ in  \eqref{eq:expinequality-3} with $\lambda^2 \sigma^2/2$ replaced by $\Psi(\lambda)$, for $\lambda_2  \in (0, \min\{-b_{-},b_{+}\})$. This results in the following bound
\begin{align}
&\Ebb_{P_W P'_Z}[l(W,Z)]-\Ebb_{P'_{Z_i}P_{W|Z_i}}[l(W,Z_i)]\leq \non \\
&\frac{1}{\lambda_2}\biggl(\Psi(\lambda_2)+\Psi(-\lambda_2)+2 D_{\KL}(P'_Z||R^{\alpha_1}_Z)+I(W;Z_i) \biggr).
\end{align}Now optimizing over $\lambda_2\in (0, \min\{-b_{-},b_{+}\})$ yields the required bound.
\end{proof}
We now specialize the bound in \eqref{eq:avgbound_generalloss} to the case of sub-gamma loss function.
\begin{corollary}
Assume that the loss function $l(W,Z)$ is sub-gamma with variance parameter $\sigma$ and scale parameter $c$ such that its CGF is upper bounded by $\Psi(\lambda)= \lambda^2 \sigma^2/2(1-c|\lambda|)$ for $|\lambda|<1/c$. Then, the following upper bound holds for the average transfer generalization gap when $\beta \in (0,1)$ and $\alpha_2 =0.5$, 
\begin{align}
& \Delta L^{\avg}\leq \frac{\gamma}{\beta M}\sum_{i=1}^{\beta M}\biggl(2\sqrt{2\sigma^2(D_{\JS}^{\alpha_1,0.5}(P'_Z||P_Z)+0.5 I(W;Z_i))}\non \\&+2c \sqrt{2\sigma^2(D_{\JS}^{\alpha_1,0.5}(P'_Z||P_Z)+0.5I(W;Z_i)} \biggr)+
\non \\&\frac{1-\gamma}{(1-\beta)M}\sum_{i=\beta M+1}^M  \biggl( 2 \sqrt{\sigma^2 (2D_{\KL}(P'_Z||R_Z^{\alpha_1})+I(W;Z_i))}\non \\& +c(2D_{\KL}(P'_Z||R_Z^{\alpha_1})+I(W;Z_i)).\biggr)
%
\end{align}
\end{corollary}
\begin{proof}
The proof follows from \eqref{eq:avgbound_generalloss} where the optimizing $\lambda_1$ is given as
$$\frac{\sqrt{ D_{\JS}^{\alpha_1,\alpha_2}(P'_Z||P_Z)+I(W;Z_i)}}{\sqrt{2\sigma^2}+ 2c\sqrt{D_{\JS}^{\alpha_1,\alpha_2}(P'_Z||P_Z) +I(W;Z_i)}} \in (0, \frac{1}{2c})$$
and the optimizing $\lambda_2$ corresponds to
$$ \lambda_2=\frac{\sqrt{2D_{\KL}(P'_Z||R_Z^{\alpha_1})+I(W;Z_i)}}{\sigma+\sqrt{2D_{\KL}(P'_Z||R_Z^{\alpha_1})+I(W;Z_i)}} \in (0, \frac{1}{c}).$$
\end{proof}
\bibliographystyle{IEEEtran}
\bibliography{ref}
\end{document}